\documentclass[letter,twoside]{article}

\usepackage[utf8]{inputenc}
\usepackage[T1]{fontenc}
\usepackage[english]{babel}
\usepackage{graphicx}
\usepackage{grffile}
\graphicspath{{figures/}}
\usepackage{subfigure}
\usepackage{float}
\usepackage{booktabs}
\usepackage{tabu}
\usepackage{rotating}
\usepackage{array}
\usepackage{multicol}
\usepackage{multirow}
\usepackage{url}
\usepackage{footnote}
\usepackage{footmisc}
\usepackage{color}
\usepackage[hidelinks]{hyperref} 
\usepackage[all]{hypcap}
\usepackage[nonumberlist,acronym,nomain,nowarn]{glossaries}
\glsdisablehyper
\makeglossaries

\usepackage{pslatex}
\usepackage{apalike}
\usepackage{SCITEPRESS}

\usepackage{graphicx,times,amsmath,amsthm,subfigure,bm,multirow,hyperref}
\newtheorem{definition}{Definition}
\newtheorem{theorem}{Theorem}

\def\a{\mathbf a}
\def\w{\mathbf w}

\def\a{\mathbf a}

\def\q{\mathbf q}

\def\s{\mathbf s}

\def\w{\mathbf w}

\begin{document}


\title{On ``Deep Learning'' Misconduct}

\author{\authorname{Juyang Weng$^{1,2,a}$}
	\affiliation{$^1$Brain-Mind Institute, 4460 Alderwood Dr. Okemos, MI 48864 USA}
	\affiliation{$^2$GENISAMA}
	\email{juyang.weng@gmail.com, jweng@genisama.com}
}


\abstract{This is a theoretical paper, as a companion paper of the plenary talk for the same conference ISAIC 2022. In contrast to the author's plenary talk in the same conference, conscious learning \cite{WengCLICCE22,WengCLAIEE22}, which develops a single network for a life (many tasks), ``Deep Learning'' trains multiple networks for each task.  Although ``Deep Learning''  may use different learning modes, including supervised, reinforcement and adversarial modes, almost all ``Deep Learning'' projects apparently suffer from the same misconduct, called ``data deletion'' and ``test on training data''.  This paper establishes a theorem that a simple method called Pure-Guess Nearest Neighbor (PGNN) reaches any required errors on validation data set and test data set, including zero-error requirements, through the same misconduct, as long as the test data set is in the possession of the authors and both the amount of storage space and the time of training are finite but unbounded.  The misconduct violates well-known protocols called transparency and cross-validation.   The nature of the misconduct is fatal, because in the absence of any disjoint test, ``Deep Learning'' is clearly not generalizable.}
\onecolumn \maketitle \normalsize \vfill


\section{\uppercase{Introduction}}
\label{SE:intro}

\noindent The problem addressed is the widespread so-called ``Deep Learning'' method---training neural networks using error-backprop.  The objective is to scientifically reason that the so-called ``Deep Learning'' contains fatal misconduct.    This paper reasons that ``Deep Learning'' was not tested by a disjoint test data set at all.  Why?  The so-called ``test data set'' was used in the Post-Selection step of the training stage.

Since around 2015 \cite{Russakovsky15}, there has been an ``explosion'' of AI papers, observed by many conferences and journals.  Many publication venues rejected many papers based on superficial reasons like topic scope, instead of the deeper reasons here that might explain the ``explosion''.  The ``explosion'' does not mean that such publication venues are of high quality (with an elevated rejection rate).  The author hypothesizes that the ``explosion'' is related to the widespread lack of awareness about the misconduct.   

Projects that apparently embed such misconducts include, but not limited to, AlexNet \cite{Krizhevsky17}, AlphaGo Zero \cite{Silver17}, AlphaZero \cite{Silver18}, AlphaFold \cite{Senior20}, MuZero \cite{Schrittwieser20}, 
and IBM Debater \cite{Slonim21}.  For open competitions with AlphaGo \cite{Silver16}, this author alleged that humans
did post-selections from multiple AlphaGo networks on the fly when test data were arriving from Lee Sedol or Ke Jie  \cite{WengMisleadAIEE23}.   More recent citations are in the author's misconduct reports submitted to {\em Nature} \cite{WengNatureReport21} and {\em Science} \cite{WengScienceReport21}, respectively.

Two misconducts are implicated with so-called ``Deep Learning'': 
\begin{description}
\item[Misconduct 1:] {\em hiding} data—the human authors hid data that look bad.   
\item[Misconduct 2:]  {\em cheating} through a test on training data—the human authors tested on training data but miscalled the reported data as ``test''. 
\end{description}
The nature of Misconduct 1 is hiding.  The nature of the Misconduct 2 is cheating.   They are the natures of actions from the authors, regardless of whether the authors intended to hide and cheat or not.   

Without a reasonable possibility to prove what was in the mind of the authors, the author does not claim that the questioned authors intentionally hid and cheated when they conducted such misconduct. 

The following analogy about the two misconducts is a simpler version in layman’s terms.  The so-called ``Deep Learning'' practice is like in a lottery scheme, a winner of a lottery ticket reports that his ``technique'' that provides a set of numbers on his lottery ticket has won 1 million dollars (Misconduct 2, since the winner has been picked up by a lucky chance after lottery drawing is finished), but he does not report how many lottery tickets he and others have tried and what is the average prize per ticket across all the lottery tickets that have tried (Misconduct 1, since he hid other lottery tickets except the luckiest ticket).  His ``technique'' will unlikely win in the next round of the lottery drawing (his ``technique'' is non-generalizable).

In the remainder of the paper, we will discuss four learning conditions in Section \ref{SE:Conditions} from which we can see that we cannot just look at superficial ``errors'' without limiting resources. 
Section~\ref{SE:4Maps} discusses four types of mappings for a learner, which gives spaces on which we can 
discuss errors. 
Post-Selections are discussed in Section~\ref{SE:Post}.  Section~\ref{SE:conclusions} provides concluding remarks.

\section{\uppercase{Four Learning  Conditions}}
\label{SE:Conditions}

Often, artificial intelligence (AI) methods were evaluated without considering how much computational resources are 
necessary for the development of a reported system.   Thus, comparisons about the performance of the systems have been biased toward competitions about how much resources a group has at its disposal, regardless how many networks have been trained and discarded, and how much time the training takes.

By definition, the Four Learning  Conditions for developing an AI system are:  (1) A body including sensors and effectors, (2) a set of restrictions of learning framework, including whether task-specific or task-nonspecific, batch learning or incremental learning; (3) a training experience and (4) a limited amount of computational resources including the number of hidden neurons.

\section{\uppercase{Four Mappings}}
\label{SE:4Maps}

Traditionally, a neural network is meant to establish a mapping $f$ from the space of input $X$ to the space of
class labels $L$, 
\begin{equation}
f:X \mapsto L
\label{EQ:XtoLmapping}
\end{equation}
\cite{Funahashi89,Poggio90a}.   $X$ may contain a few time frames.  

For temporal problems, such as video analysis problems, speech recognition problems, and computer game-play problems, we can include context labels in the input space, so as to learn a mapping 
\begin{equation}
f:X \times L \mapsto L.
\label{EQ:XLtoLmapping}
\end{equation}
where $\times$ denotes the Cartesian product of sets. 

The developmental approach deals with space and time in a unified fashion using a neural network such as Developmental Networks (DNs) \cite{WengWhy11} 
whose experimental embodiments range from Where-What Network WWN-1 to Where-What Network WWN-9.  The DNs went beyond vision problems to attack general AI problems including vision, audition, and natural language acquisition as emergent Turing machines \cite{WengIJIS15}.
DNs overcome the limitations of the framewise mapping in Eq.~\eqref{EQ:XLtoLmapping} by dealing with lifetime mapping without using any symbolic labels:
 \begin{equation}
f: X(t-1)\times Z(t-1) \mapsto  Z(t), t=1, 2, ... 
\label{EQ:XZmapping}
\end{equation}
where $X(t)$ and $Z(t)$ are the sensory input space and motor input-output space, respectively.

Consider space: Because $X$ and $Z$ are vector spaces of sensory images and muscle neurons, we need internal 
neuronal feature space $Y$ to deal with sub-vectors in $X$, $Z$ and their hierarchical features.  

Consider time: Furthermore, considering symbolic 
Markov models, we also need further to model how $Y$-to-$Y$ connections
enable something similar to higher and dynamic order of time in Markov models.  With the two considerations Space and Time, the above lifetime mapping in Eq.~\eqref{EQ:XZmapping} is extended to:
\begin{equation}
f: X(t-1)\times Y(t-1) \times Z(t-1) \mapsto  Y(t) \times Z(t),
\label{EQ:XYZmapping}
\end{equation} 
$ t=1, 2, ... $ in DN-2.  It is worth noting that the $Y$ space is inside a closed ``skull'' so it cannot be directly supervised. 
$Z(t-1)$ here is extremely important since it corresponds to the state of an emergent Turing machine.  

In performance evaluation of the developmental approach, all the errors occurring during any time in Eq.~\eqref{EQ:XYZmapping} of each life are 
recorded and taken into account in the performance evaluation.  This is in sharp contrast with, and free from, Post-Selection.  

\section{\uppercase{Post-Selections}}
\label{SE:Post}

\begin{definition}[Training and test stages]
\label{DF:2stages}
Suppose that the development of a classification system $f$ is divided into two stages, a training stage and a test stage, where the training stage must provide a completely trained system so that given any new input $\q$, not in the possession of the human trainer, the trained system must provide top-$m$ labels in $L$ (e.g., $m=5$ for top 5) for the input $q$.
\end{definition}

Let us discuss three types of errors.

\subsection{Fitting Errors}

Given an available data set $D$, $D$ is partitioned into three mutually disjoint sets, a fitting set $F$, a validation set $V$ (like a mock exam), and a test set $T$ so that
\begin{equation}
D=F\cup V \cup T.
\label{EQ:disjoint}
\end{equation}
Two sets are disjoint if they do not share any elements.   
The validation set is possessed by the trainer, but the test set should not be possessed by the trainer since the test should be conducted by an independent agency.  Otherwise, $V$ and $T$ become equivalent. 

Typically, we do not know the hyper-parameter vector $\a$ (e.g., including receptive fields of neurons), thus the so-called ``Deep Learning'' technique searches for $\a$ as $\a_i$, $i=1, 2, ... k$.  Given any hyper-parameter vector $\a_i$, it is unlikely that a single network initialized by a set of random weight vectors $\w_j$ can 
result in an acceptable error rate on the fitting set, called fitting error.   The error-backprop  training intends to minimize the error locally along the gradient direction of $\w_j$.
That is how the multiple sets of random weight hyper-parameter vectors come in.  For $k$ hyper-parameter vectors $\a_i$, $i=1, 2, ... k$ and $n$ sets of random initial weight vectors $\w_j$, the error back-prop training results in $kn$ networks 
\[
\{N(\a_i, \w_j) \;|\; i=1, 2, ... , k, j=1, 2, ..., n\} .
\]
Error-backprop locally and numerically minimizes the fitting error $f_{i,j}$ of $N(\a_i, \w_j)$ on the fitting set $F$.  

\subsection{Abstraction Errors}

The effect of abstraction error can be considered as a lack of the degree of abstraction.    

One effect is the genome, or the developmental program.  As monkeys do not have a human vocal tract to speak human languages and their brains are not as large as human brains, monkeys cannot abstract using human languages.   

Another effect is age.  If one's  learning is not sufficient (e.g., too young), a human child's brain network has not learned the required abstract concepts.   Such abstract concepts include where, what, scale, many other concepts one learns in school, as well as lifetime concepts such as better education leads to a more productive life.   Therefore, a
young child cannot do well for jobs that require a human adult.   We call the error as post error, since it is the error after (i.e., post) a certain amount of training.  

\begin{figure}[h]
	\centering
	\includegraphics[width=1.0\linewidth]{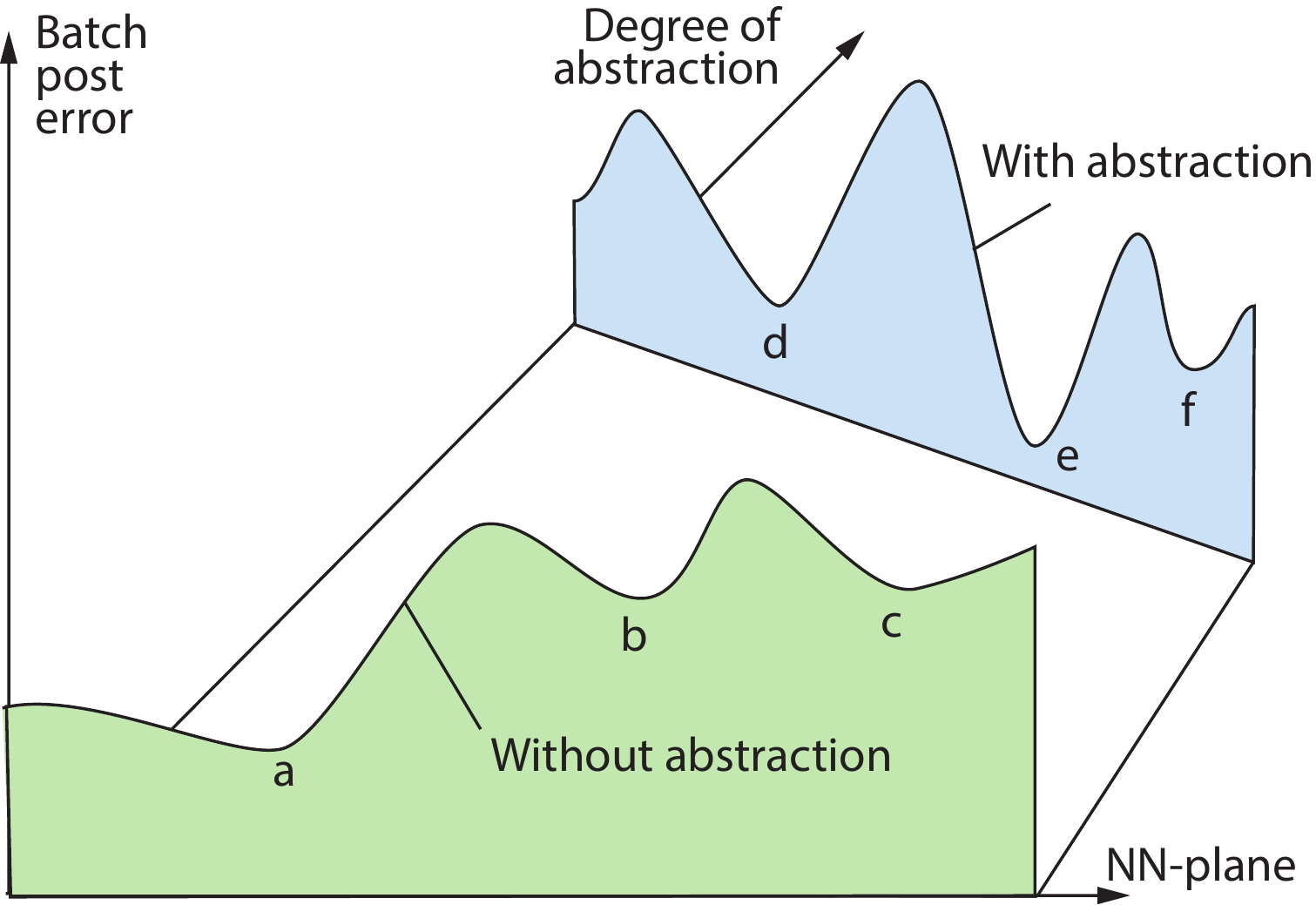}
	\caption{The effect of degree of abstraction by a neural network.  The horizontal axis indicates the possible value of the parameters of a neural network, denoted as NN-plane.  The 1-D here corresponds to the 60-million dimension in \cite{Krizhevsky17}.  The vertical axis is the batch post-selection error of the corresponding trained network.  }
	\label{FG:PSUTS-1D-Abs-v-NN-plane}
\end{figure}

Fig.~\ref{FG:PSUTS-1D-Abs-v-NN-plane} gives a 1D illustration for the effect of abstraction.   If the architecture of a neural network is inadequate
e.g., pure classification through data fitting in Eq.~\eqref{EQ:XtoLmapping}, the manifold of post error corresponds to that of ``without abstraction'' (green) in Fig.~\ref{FG:PSUTS-1D-Abs-v-NN-plane}.   Different 
positions along the horizontal axis (NN-plane) correspond to different parameters of the same type of neural networks.   The lowest point on the green manifold is labeled ``a'', but as we will see below, the smallest post error is missed by all error backprop methods since it typically does not coincide with a pit in the training data set.   This is true because the test set $T$ and the fitting set $F$ are disjoint, but the fitting error is based on the fitting set $F$ but the post error is based on the test set $T$.

In Fig.~\ref{FG:PSUTS-1D-Abs-v-NN-plane}, the blue manifold is a better than the green manifold, because the lowest point ``e'' on the blue manifold is lower than the lowest point ``a'' on the green manifold.  They correspond to different mapping parameter vector definitions for $\a$.  For example, the green manifold and the blue manifold correspond to Eq.~\eqref{EQ:XtoLmapping} and Eq.~\eqref{EQ:XZmapping}, respectively.  

Note that Fig.~\ref{FG:PSUTS-1D-Abs-v-NN-plane} only considers batch post errors in batch learning, but Eq.~\eqref{EQ:XZmapping} and Eq.~\eqref{EQ:XYZmapping} deal with incremental learning.

Given a defined architecture parameter vector $\a$, each searched $\a_i$ as a guessed $\a$ will also give a very different manifold in Fig. ~\ref{FG:PSUTS-1D-Abs-v-NN-plane}, where for simplicity, the manifold is drawn as a line.   In general, the worse a guessed $\a_i$ is, the higher the corresponding position on the  manifold but the amounts of increase at different points of the manifold are not necessarily the same since the manifold depends also on the test set $T$. 

From this point on, we assume that the architecture parameters $\a$ have been pre-defined as the so-called hyper-parameter vector, but their vector 
values are unknown.  The components in $\a$ may include the number of layers and the receptive field size of each layer.  But we need to realize that age, environment, and teaching experience greatly change
the landscape in Fig. ~\ref{FG:PSUTS-1D-Abs-v-NN-plane}, as we discussed in Sec.~\ref{SE:4Maps}.  

\subsection{Validation and Test Errors}

\begin{figure*}
	\centering
	\includegraphics[width=0.7\linewidth]{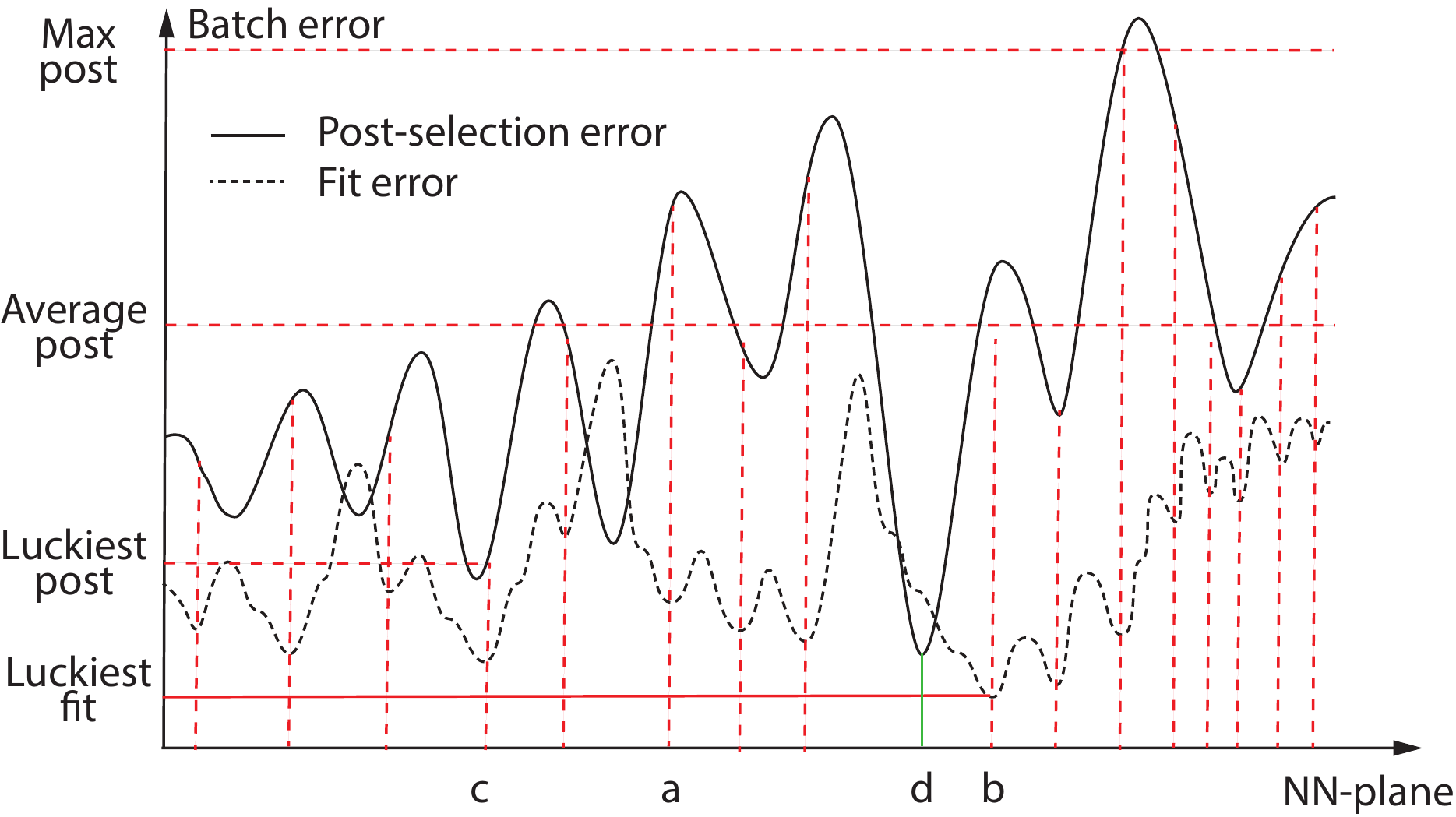}
	\caption{A 1D-terrain illustration for the fitting error (dashed curve) from the fitting data set and the post-selection error (solid curve) from the test data set.  The luckiest post-selection error $d$ is missed because it is not near a pit of dashed curve.   In the misconduct, only the luckiest post at point $c$ is reported, but at least the average post error and the maximum post error should also be reported.  The test data set was unethically used to find the luckiest post in the training stage.}
	\label{FG:PSUTS-Fit-Error-1D}
\end{figure*}

Suppose we train a ``Deep Learning'' neural network $N(\a_i, \w_j)$ using error-backprop or reinforcement learning (a local gradient-based method), starting from $\a_i$ and $\w_j$.   

The following explanation of the two misconducts is also in layman’s terms but is more precise.   The so-called ``Deep Learning'' technique is like finding a location of many ``Neural Network'' balls using ``oil wells'' data.  

An oil well is a drill hole boring in the Earth that is designed to bring petroleum oil hydrocarbons to the surface.   We use the term ``oil well'' to indicate that, like drilling ``oil wells'', it is costly to collect and annotate data.  Like ``oil wells'', any data set $D$ is always very sparse on the NN-plane.  

Each ``oil well'' data contains a location on the NN-plane and an error height (how good the ``oil well'' is or how well the neural network at the NN-plane location fits the corresponding data that were used to construct the terrain).  All available ``oil wells'' data are divided into two disjoint sets, a ``fit'' set and a so-called ``test'' set $P=V\cup T$.  

The training stage of the technique has two steps, the ``fitting'' step and the ``post-selection'' step.  The fitting step uses the ``fit'' set $F$.  The ``post-selection'' step uses 
the so-called ``test'' set $P=V\cup T$, but this is wrong because the second step of the training stage must not use the so-called ``test'' set $T$.  Consequently, almost all ``Deep Learning'' techniques have not been tested at all, as Theorem~\ref{TM:without} below will establish. 

A real-world plane has a dimension of 2, but a Neural Network plane, NN-plane, has a dimension of at least millions, corresponding to millions of parameters to be learned by each ``Neural Network'' ball.  (The original ``Deep Learning'' paper
\cite{Krizhevsky17} has a dimensionality of 200B.) The ``fit'' set and ``test'' set correspond to two heights at each of all possible locations on the ``NN'' plane, called, respectively, the ``fit'' error and post-selection error that was miscalled ``test error'' by \cite{Krizhevsky17}.  Although ``oil wells'' are costly, the more  Neural Network balls technique tries, the better chance to find a lucky ball whose final location has a low post-selection error.  

In the ``fitting'' step, the technique drops many balls to many random locations of the NN-plane, typically many more than the dimension of the NN-plane. From random locations that the balls landed at, all the balls automatically roll down (according to its height or another artificial ``reward'') until they get stuck in a local pit and then they stop.
 
In Fig.~\ref{FG:PSUTS-Fit-Error-1D}, the NN-plane is illustrated as a horizontal line.   Balls roll down on the dashed-line terrain.  If we drop only one ball, it might stop at location $a$ whose fit height is mediocrely low.  If we drop three balls, they may stop at locations $a$, $b$ and $c$, respectively.  If we drop even more balls, we assume that all the vertical dashed lines have at least one
ball that stopped.  The fitting step missed location $d$, the lowest post-selection error possible, because it is not in a pit. 

In the Post-Selection step: record the so-called ``test'' height of every ball at its stopped NN-plane location.  Misconduct 2 means to ``post-select'' the luckiest ball whose ``test'' height is the lowest among all randomly tried balls.  Only this luckiest ball was reported to the public.  Misconduct 1: All less lucky balls are discarded because their ``test'' heights look bad. 

In Fig.~\ref{FG:PSUTS-Fit-Error-1D}, so-called test height is indicated by solid-line terrain but it should be called post-selection error instead.   Generally, the solid-line terrain can cross under the dashed-line terrain, but it is unlikely at a pit (see $d$) because (1) the Fit procedure greedily fits the 
fitting set $F$ using error-backprop but does not fit $V$ or $T$, (2) $F$, $V$ and $T$ all have many samples.  

Because the ``post-selection'' step is within the training stage and the ``test'' data are all used in the training stage, this technique corresponds to Misconduct 2 (test on training data).  Although all the balls have not ``seen'' the so-called ``test'' data when they roll down a hill, they all have ``seen'' the ``test'' data during the post-selection step of the training stage.  

The reported luckiest ball is not generalizable to a new test due to the two alleged misconducts, Misconduct 1: the technique hides many random networks that are bad; Misconduct 2:  the technique cheats: the miscalled ``test'' error is actually a ``training'' error.

Fig.~\ref{FG:PSUTS-Fit-Error-1D} indicates post-selection errors as horizontal dashed lines.  We can see that 
at least the maximum post-selection error (Max post) and average post-selection error (average post) should be reported, not just the luckiest post-selection error (luckiest post).  The $k$-fold cross-validation protocol 
\cite{DudaHartStork} further requires that the roles of the fit set and the test set be switched by dividing all available data $D$ into $k$ folds of disjoint subsets.  

Because the test set was used in the training stage, 
Fig.~\ref{FG:PSUTS-Fit-Error-1D} corrects the so-called ``test'' errors as post-selection errors. 

We define a simple system that is easy to understand for our discussion to follow. 
Consider a highly specific task of recognizing patterns inside the annotated windows in Fig.~\ref{FG:ImageNet-Annotation}.  This is a simplified case of the three tasks---recognition (yes or no, learned patterns at varied positions and scales), detection (presence of, or not, learned patterns) and segmentation (of recognized patterns from input).  These three tasks of natural cluttered scenes were dealt with, for the first time, by  
the first ``Deep Learning'' network for 3D---Cresceptron~\cite{WengCresIJCV97}.  

Cresceptron only trains one network for each task.  But later  ``Deep Learning'' networks train multiple networks for each task and then use Post-Selection to select fewer networks.  Below, by ``Deep Learning", we mean such 
Post-Selection based networks.  Later data sets like ImageNet \cite{Russakovsky15} contain many more image samples but we will see below that all ``Deep Learning" networks simply fit the training data, validation data and test data and are without a test at all.
\begin{figure}[tb]
	\centering
	\includegraphics[width=1.0\linewidth]{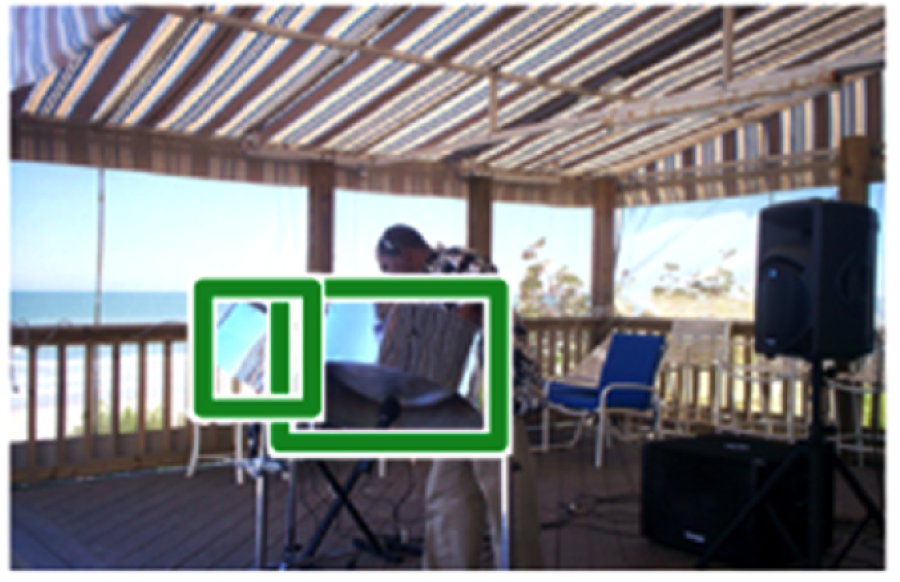}
	\caption{
	ImageNet-like annotation.  Two annotated windows as two training samples in each cluttered image. 
The ImageNet competitions extend to more positions and scales than such windows.  Courtesy of \cite{Russakovsky15}.
	}
	\label{FG:ImageNet-Annotation}
\end{figure}

\subsection{Pure-Guess Nearest Neighbor}  
\cite{WengMisleadAIEE23} proposed a Nearest Neighbor With Threshold (NNWT) method to establish that such a simple
classifier beats all Post-Selection based ``Deep Learning'' methods since it satisfies even a zero-requirement on 
both the validation error and the test error, using the two misconducts.
Here, to be clearer, the following new PGNN is without the threshold. 

\begin{definition}[Pure-Guess Nearest Neighbor, PGNN]
\label{DF:NN}
PGNN method stores all available data $D$, the fit set $F$, the validation set $V$ and the miscalled test set $T$. 
To deal with the ImageNet Competitions in Fig.~\ref{FG:ImageNet-Annotation}, the method uses the window-scan method  probably first proposed by Cresceptron \cite{WengCresIJCV97}.  Given the query input from every scan window, PGNN finds its nearest-neighbor sample in $D$ and outputs the stored label. PGNN perfectly fits $F$. 
For samples in $V$ and $T$, PGNN randomly and uniformly guesses a label using Post-Selection and stores the guessed label. 
\end{definition}  
 
From a fit set $F$ and a Post set $P=V\cup T$, the PGNN algorithm is denoted as $G(F, P, e)$, where $e$ is a seed for a pseudo-random number.  

The training stage of PGNN:  

The 1st step  $\mbox{Fit}(F, B)$: Store the entire fitting set $F=\{ (\s, l)\}$ into database $B$, where $\s$ and $l$ are the normalized sample and the label from the corresponding annotated window $\w$, respectively.  The window scan has a pre-specified position range for row and column $(r, c)$ of a scan window, and a pre-specified range for the scale of the window.   The window scan tries all the locations and all the scales in the pre-specified ranges.  For each window $\w$, Fit crops the image at the window and normalizes the cropped image into 
a standard sample $\s$.  All standard samples in $B$ have the same dimension as a vector in row-major storage.  

The 2nd step $\mbox{Post}(P, L, e, B)$:  From every query image $\q\in P$, for every scan window $\w$ for $\q$, compute its standard sample $\s$.  If $\s$ is new, guess a label $l$ for $\s$ to generate $(\s, l)$, where $l$ is randomly sampled from $L$ using a uniform distribution, identically independently distributed.   Store $(\s, l)$ into database $B$.  While $\mbox{Post}(P, L, e, B)$ is not good enough on $P$, run $\mbox{Post}(P, L, e, B)$ using the returned new seed $e$.  

Each run of Post corresponds to a new ``Deep Learning'' network where each network starts from a new random set of weights and a new set of hyper parameters. 

The performance stage of PGNN:  

$\mbox{Run}(P,B)$: For every query image $\q\in P$, for every scan window $\w$ for $\q$, compute its standard sample $\s$.  Find its nearest sample $\s^* $ from $B$, output the stored label $l$ associated with the nearest neighbor $\s^*$.

PGNN uses a lot of space and time resources for over-fitting $F$ and $P$.  It randomly guesses labels for $P=V\cup T$ until all the guesses are correct.  Therefore, it satisfies the required error for $V$ and $T$, as long as the human annotation is consistent.    PGNN here is slower than NNWT in \cite{WengMisleadAIEE23} which interpolates from samples in $F$ until the distance is beyond the
threshold (a hyper-parameter).  But PGNN is simpler for our explanation of misconduct since it drops the threshold in NNWT.  

To understand why Post-Selections is misconduct that gives misleading results, let us derive the following important theorem.
\begin{theorem}[PGNN Supremacy]
\label{TM:Supremacy}
Given any validation error rate $e_v\ge 0$ and test error rate $e_t\ge 0$, using Post-Selections the {\rm PGNN} classifier satisfies any required $e_v$ and $e_t$, if the author is in the possession of the test set $P=V\cup T$  and both the storage space and the time spent on the Post-Selections are finite but unbounded, if the Post-Selection is allowed.
\end{theorem}

\begin{proof}
Because the number of seeds to be tried 
during the Post-Selection is finite but unbounded, we can prove that there is a finite time at which a lucky seed $s$ will produce a good enough verification error and test error.  
Although the waiting time is long, the time is finite because $V$ and $T$ are finite.  Let us formally prove this.  Suppose $l$ is the number of labels in the output set $L$.  For the set of queries in $V$ and $T$, there are $k$ (constant) outputs that must be guessed.  The probability for a single guess to be correct is $1/l_0$, $l_0=\| L \|$, due to the uniform guess in $L$.  The probability for $k$ guesses to be all correct 
is  $(1/l_0)^k=1/l_0^k$ because guesses are all mutually independent.  The probability to guess at least one label wrong is 
$1- 1/l_0^k$, with $0< 1- 1/l_0^k< 1$.  The probability for as many as $n$ runs of $\mbox{Post}$, all of which do not satisfy the $e_v =0$ and $e_t=0$, is
\[
p(n) = (1- 1/l_0^k)^n \longrightarrow 0, 
\]
as $n$ approaches infinity, because $0< 1- 1/l_0^k< 1$.  Therefore, within a finite time span, a process of trying incrementally more networks using Post will get a lucky network that satisfies both the required $e_v $ and $e_t $.   This is the luckiest network from the Post-Selection.
\end{proof}

Theorem~\ref{TM:Supremacy} has established that Post-Selections can even produce a superior classifier that gives any required validation error and any test error, including zero-value requirements! Yes, while the test sets are in the possession of authors, the authors could show any superficially impressive validation error rates and test error rates (including even zeros!) because they used Post-Selections without a limit on resources (to store all data sets and to search for the luckiest network).
It is of course time consuming for a program to search for a network whose guessed labels are 
good enough. But such a lucky network will eventually come within a finite time frame!   

\subsection{Absence of Test}

Does the Post-Selection step belong to the training stage?

\begin{theorem}[Post-Selection]
\label{TM:Post-Selection}
Between the two stages, training and test, the Post-Selection step that selects $m$ required networks from $n>m$ networks (e.g., $m=5$ and $n=10000$) belongs to the training stage.
\end{theorem}
\begin{proof}
Let us prove by contradiction using Definition~\ref{DF:2stages}.   We hypothesize that the conclusion is not true, then the Post-Selection step that post-selects from $n>m$ networks belongs to the test stage.  Then in the absence of the Post-Selection step, after being given any query $\q$, the training stage is not able to produce only top-$m$  labels, but instead $n-m>0$ labels than required.  This is a contradiction to Definition ~\ref{DF:2stages}.  This means that the conclusion is correct.
\end{proof}

\begin{theorem}[Without test]
\label{TM:without}
Different from Cresceptron which trains only one network, a ``Deep Learning'' method that trains $n\ge 2$ networks and uses the so-called test set $T$ in the Post-Selection step to down-select $m <n$ networks from $n$ networks is without a test stage.
\end{theorem}
\begin{proof}
This is true because $T$ is already used in the training stage according to Theorem~\ref{TM:Post-Selection}.   
\end{proof}

The above theorem reveals that almost all so-called ``Deep Learning'' methods cited in this paper, including more in \cite{WengNatureReport21,WengScienceReport21}, in the way they published, were not tested at all. 
The basic reason is that the so-called test set $T$ was used in the training stage.
Because ``Deep Learning'' is not tested, the technique is not trustable.  

 A published so-called ``Deep Learning'' paper \cite{GaoBEAN21} claimed to use an average ``test'' error during the Post-Selection step of the training stage. It reported a drastically worse performance, 12\% average error on the MNIST data set instead of 0.23\% error that uses the luckiest (MNIST website). 
 12\% is over 52 times larger than 0.23\%.    \cite{GaoBEAN21} still contains Misconduct 2: The average is only across a partial dimensionality of the NN-plane $\w$ but other remaining dimensionality $\a$ of the NN-plane till uses the ``luckiest''.  This quantitative information supports that so-called ``Deep Learning'' technology is not trustable in practice.
 
Therefore, the published ``Deep Learning'' methods cheated and hid.
``Deep learning'' tested on a training set as \cite{DudaHartStork} warned against but miscalled the activities as ``test'' and deleted or hid data that looked bad. 

\section{\uppercase{Conclusions}}
\label{SE:conclusions}

The simple Pure-Guess Nearest Neighbor (PGNN) method beats all ``Deep Learning'' methods in terms of the superficial errors using the same misconduct.  
Misconduct in ``Deep Learning'' results in performance data that are misleading.  Without a test stage, ``Deep Learning'' is not generalizable and not trustable.   Such misconduct is tempting to those authors where the test sets are in the possession of the authors and also to open-competitions where human experts are not explicitly disallowed to interact with the ``machine player'' on the fly.  This paper presents 
scientific reasoning based on well-established principles---transparency and cross-validation.  It does not present detailed evidence of every charged paper in \cite{WengNatureReport21,WengScienceReport21}. More detailed evidence of such misconduct is referred to Weng et al. v. NSF et al. U.S. West Michigan District Court case number 1:22-cv-998.   

The rules of ImageNet \cite{Russakovsky15} and many other competitions seem to have encouraged the Post-Selections discussed here.   Even if the Post-Selection is banned, any comparisons without an explicit limit on, or an explicit comparison about, storage and time spent are meaningless. ImageNet \cite{Russakovsky15} and many other competitions did not ban Post-Selections, nor did they limit or compare storage or time.  

The Post-Selection problem is among the 20 million-dollar problems solved conjunctively by this author \cite{Weng20M22}.  Since such a fundamental problem is intertwined with other 19 fundamental problems for the brain, it appears that one cannot solve the misconduct problem (i.e., local minima) without solving all the 20 million-dollar problems altogether.

\bibliographystyle{apalike}
{\small \bibliography{shoslifref}}

\end{document}